\documentclass{article}
\usepackage{arxiv}
\usepackage[utf8]{inputenc} 
\usepackage[T1]{fontenc}   
\usepackage{hyperref}    
\usepackage{url}        
\usepackage{booktabs}   
\usepackage{wrapfig}
\usepackage{amsfonts}     
\usepackage{nicefrac}  
\usepackage{microtype}   
\usepackage[dvipsnames, svgnames, x11names]{xcolor} 
\usepackage{tcolorbox}
\usepackage{graphicx}
\usepackage{bbding}
\usepackage{subcaption}
\usepackage{wrapfig}
\usepackage{bm}
\usepackage{listings}
\usepackage[ruled,vlined]{algorithm2e}

\usepackage{amsmath}
\usepackage{amssymb}
\usepackage{mathtools}
\usepackage{amsthm}
\usepackage{multirow}
\usepackage[capitalize,noabbrev]{cleveref}
\theoremstyle{plain}
\newtheorem{theorem}{Theorem}[section]

\newtheorem{corollary}[theorem]{Corollary}
\theoremstyle{definition}
\newtheorem{definition}[theorem]{Definition}

\theoremstyle{remark}
\newtheorem{remark}[theorem]{Remark}

\usepackage{accents}
\usepackage[textsize=tiny]{todonotes}

\usepackage[utf8]{inputenc}
\usepackage[english]{babel}
\usepackage{mathrsfs}
\usepackage{times}
\usepackage{latexsym}
\usepackage{comment}

\usepackage[]{mathtools}   
\def\$#1\${\begin{align*}#1\end{align*}}
\usepackage{enumitem}
\usepackage[]{amsthm} 
\usepackage[]{amssymb}
\usepackage{bbm}
\usepackage{xcolor}
\usepackage{colortbl}
\definecolor{best}{HTML}{BAFFCD}
\definecolor{issue}{HTML}{FFC8BA}
\definecolor{bad}{HTML}{FFC87C}

\usepackage{amsmath,amsfonts,bm,amsthm,dsfont}

\def\eqref#1{equation~\ref{#1}}

\def\1{\bm{1}}
\def\0{\bm{0}}

\DeclareMathAlphabet{\mathsfit}{\encodingdefault}{\sfdefault}{m}{sl}
\SetMathAlphabet{\mathsfit}{bold}{\encodingdefault}{\sfdefault}{bx}{n}

\usepackage{tcolorbox}
\definecolor{grey}{rgb}{0.33, 0.33, 0.33}

\newcommand{\squishlist}{
\begin{list}{{{\small{$\bullet$}}}}
{\setlength{\itemsep}{1pt}      \setlength{\parsep}{0pt}
\setlength{\topsep}{-2pt}       \setlength{\partopsep}{0pt}
\setlength{\leftmargin}{1em} \setlength{\labelwidth}{1em}
\setlength{\labelsep}{0.5em} } }
\newcommand{\squishend}{  \end{list}  }

\makeatletter
\renewcommand*\env@matrix[1][*\c@MaxMatrixCols c]{%
  \hskip -\arraycolsep
  \let\@ifnextchar\new@ifnextchar
  \array{#1}}
\makeatother

\def\tr{\mathop{\text{tr}}\kern.2ex}

\long\def\comment#1{}

\def\tr{\mathop{\text{Tr}}}

\newcommand{\bel}{\begin{eqnarray}\label}
\newcommand{\eel}{\end{eqnarray}}
\newcommand{\bes}{\begin{eqnarray*}}
\newcommand{\ees}{\end{eqnarray*}}

\usepackage{xcolor}

\ifodd 1

\newcommand{\clar}[1]{\textbf{\color{green}(NEED CLARIFICATION: #1)}}
\newcommand{\response}[1]{\textbf{\color{magenta}(RESPONSE: #1)}}
\else

\newcommand{\com}[1]{}
\newcommand{\clar}[1]{}
\newcommand{\response}[1]{}
\fi
\newcommand{\RNum}[1]{\uppercase\expandafter{\romannumeral #1\relax}}

\title{Dynamics-Aware Loss for Learning with Label Noise}

\author{Xiu-Chuan Li$^{1}$ \ \ \ \ \ Xiaobo Xia$^{2}$ \ \ \ \ \ Fei Zhu$^{1}$ \ \ \ \ \ Tongliang Liu$^{2}$ \ \ \ \ \ Xu-Yao Zhang$^{1}$ \ \ \ \ \ Cheng-Lin Liu$^{1}$ \\
$^1$Institute of Automation, Chinese Academy of Science \quad\quad
$^2$The University of Sydney \\
\texttt{lixiuchuan2020@ia.ac.cn} \ \ \ \ \ 
\texttt{xxia5420@uni.sydney.edu.au} \ \\
\texttt{zhufei2018@ia.ac.cn} \ \ \ \ \  
\texttt{tongliang.liu@sydney.edu.au} \ \\
\texttt{xyz@nlpr.ia.ac.cn} \ \ \ \ \  
\texttt{liucl@nlpr.ia.ac.cn} \ \\
}
\begin{document}

\maketitle

\newcommand{\myPara}[1]{\vspace{.05in}\noindent\textbf{#1}}

\begin{abstract}
	Label noise poses a serious threat to deep neural networks (DNNs). Employing robust loss functions which reconcile fitting ability with robustness is a simple but effective strategy to handle this problem. However, the widely-used static trade-off between these two factors contradicts the dynamics of DNNs learning with label noise, leading to inferior performance. Therefore, we propose a dynamics-aware loss (DAL) to solve this problem. Considering that DNNs tend to first learn beneficial patterns, then gradually overfit harmful label noise, DAL strengthens the fitting ability initially, then gradually improves robustness. Moreover, at the later stage, to further reduce the negative impact of label noise and combat underfitting simultaneously, we let DNNs put more emphasis on easy examples than hard ones and introduce a bootstrapping term. Both the detailed theoretical analyses and extensive experimental results demonstrate the superiority of our method. Our source code can be found in \textit{https://github.com/XiuchuanLi/DAL}.
\end{abstract}

\section{Introduction}

Deep neural networks (DNNs) have achieved tremendous success in various supervised learning tasks, however, this success relies heavily on correctly annotated large-scale datasets. Collecting a large and perfectly annotated dataset can be both expensive and time-consuming, leading many to opt for cheaper methods like querying search engines, which inevitably introduces label noise. Unfortunately, the work of~\cite{overfit} has empirically proved that DNNs can easily fit an entire training dataset with any ratio of noisy labels, resulting in poor generalization performance eventually. Therefore, developing robust algorithms against label noise for DNNs is of great practical significance.

One effective and scalable solution for handling label noise is to use robust loss functions~\cite{gce, tce, js}. These methods have the advantage of simplicity. Specifically, they require no extra information such as the noise rate or a clean validation set, and incur no additional memory burden or computational cost. Cross entropy (CE) is observed to result in serious overfitting in the presence of label noise because it provides DNNs with strong fitting ability. Meanwhile, although mean absolute error (MAE) is theoretically robust against label noise~\cite{mae}, it suffers from severe underfitting in practice, especially on complex datasets. In light of this, many robust loss functions have been proposed to reconcile fitting ability with robustness, among which the generalized cross entropy (GCE)~\cite{gce} is the most representative method. As shown in Figure~\ref{fig_motivation}, selecting a suitable value for $q \in (0,1)$ such as 0.7, GCE outperforms both CE and MAE by a large margin.

Besides devising robust learning algorithms, some work aimed to delve deep into the behavior of DNNs during the process of learning with label noise.  The work of~\cite{dynamic} observed that while DNNs are capable of memorizing noisy labels perfectly, there are noticeable differences in DNNs' learning status at different stages of the training process. Specifically, DNNs tend to first learn patterns shared by the majority of training examples which benefit generalization, and then gradually overfit label noise which harms generalization. We refer to this as the dynamics of DNNs in the following. Further evidence of this observation is provided by~\cite{dynamic2} which showed that DNNs first learn simple and generalized representations via subspace dimensionality compression, then memorize noisy labels through subspace dimensionality expansion. 

\begin{figure*}[t]
	\centering
	\includegraphics[width= \textwidth]{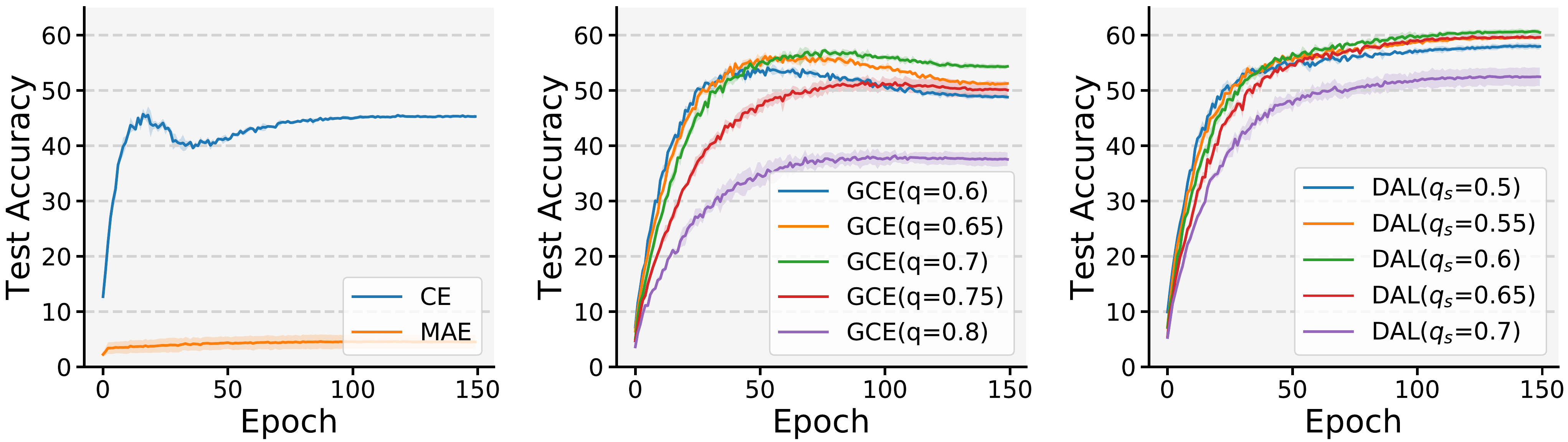}
	\caption{Performance on CIFAR-100 with 40\% instance-dependent noise. Left: CE overfits label noise while MAE suffers from serious underfitting. Middle: GCE with proper $q$ outperforms both CE and MAE by a large margin. Right: The performance of DAL both increases quickly at the early stage and grows steadily afterwards, which is remarkably better than GCE. Moreover, DAL is much less sensitive to its hyper-parameter than GCE.}
	\label{fig_motivation}
\end{figure*}

It is clear from the above views that there exists a discrepancy between the statics of robust loss functions and the dynamics of DNNs learning with label noise, resulting in inferior performance. Specifically, with a static trade-off between fitting ability and robustness, the generalization fails to both improve rapidly at the early stage and maintain steady growth thereafter. As demonstrated in Figure~\ref{fig_motivation}, if GCE provides DNNs with relatively strong fitting ability, e.g. $q \leq 0.7$, the test accuracy rises quickly at the early stage when DNNs tend to learn beneficial patterns, but it experiences a remarkable decline at the later stage, indicating that DNNs overfit label noise. In contrast, if GCE provides DNNs with strong robustness, e.g. $q>0.7$, although the test accuracy grows steadily during the whole training process, it always remains at a relatively low level. Moreover, the performance of GCE is quite sensitive to $q$, which makes it less applicable in the real world.

Fortunately, a loss function with a dynamic trade-off between fitting ability and robustness can solve the above problem. Specifically, we prioritize fitting ability to rapidly increase the classification accuracy at the early stage, then gradually improve robustness to ensure steady performance growth thereafter. Furthermore, to further reduce the negative impact of label noise, we place greater emphasis on easy examples compared to hard ones at the later stage, as easy examples are more likely to be correctly labeled~\cite{coteaching,coteaching+}. However, although strong robustness can prevent overfitting, it may lead to underfitting, especially when the noise rate is high, so we incorporate a bootstrapping term at the later stage that encourages DNNs to disagree with the provided labels on some examples. We refer to the proposed method dynamics-aware loss (DAL) and exhibit its performance in Figure~\ref{fig_motivation}. With $q_s=$0.6, its performance could both rise quickly at the early stage and grow steadily later on, resulting in a substantial improvement compared to GCE. Moreover, DAL is much less sensitive to its hyper-parameter $q_s$ compared to GCE. As shown in Figure~\ref{fig_motivation}, varying $q_s$ in [0.55, 0.65] only results in a performance fluctuation of approximately 1\%.

In summary, our key contributions include
\begin{itemize}
	\item We address the discrepancy between the static robust loss functions and the dynamics of DNNs learning with label noise by proposing the Dynamics-Aware Loss (DAL), which empoloys a dynamic trade-off between fitting ability and robustness. Besides, DAL places greater emphasis on easy examples over hard ones and introduces a bootstrapping term at the later stage of the training process, which mitigate the negative impact of label noise and combat underfitting.
	\item We present detailed analyses to certify the superiority of DAL without any assumption about the label noise model. The model trained with DAL can achieve better generalization theoretically.
	\item Our extensive experiments on various benchmark datasets demonstrate the superior performance and practicality of DAL compared to other robust losses without requiring extra information or resources. We also empirically verify that it is complementary to other robust algorithms for learning with label noise, and that it helps to improve backdoor robustness.
\end{itemize}

The rest of the paper is organized as follows. In Section~\ref{sec_work}, we give a brief review of related work on learning with label noise. In Section~\ref{sec_method}, we introduce our proposed DAL in detail. In Section~\ref{sec_theory}, we provide detailed theoretical analyses. In Section~\ref{sec_experiment}, we provide extensive experimental results. Finally, we conclude the paper in Section~\ref{sec_conclusion}.

\section{Related Work} \label{sec_work}

In this section, we briefly summarize existing work in the field of learning with label noise. For a more detailed discussion please refer to~\cite{survey}.

\subsection{Noise Transition Matrix Estimation}
In theory, the clean class posterior can be inferred by combining the noisy class posterior and the noise transition matrix that reflects the label flipping process. Accurate estimation of the noise transition matrix is crucial for building statistically consistent classifiers, which converge to the optimal classifiers derived from clean data. However, a large estimation error in the noise transition matrix can significantly degrade the classification accuracy, and most methods presuppose high-quality features such as anchor points. To mitigate the above limitations, many studies have focused on reducing the estimation error of the noise transition matrix~\cite{estimation4, estimation5} and reducing the requirement of high-quality features~\cite{anchor, beyond}.

\subsection{Loss Adjustment}
These methods adjust the loss of each training example before back-propagation, which could be further divided into loss reweighting~\cite{meta, reweighting3} and label correction~\cite{correction3, correction4}. The former assigns smaller weights to the potentially incorrect labels while the latter replaces the given label with a combination of itself and the prediction. Loss reweighting is usually realized by meta-learning that learns weights for each sample, which requires some clean data to train the meta DNN. Label correction may overfit wrongly refurbished labels if there are many confusing classes in the training dataset.

\subsection{Sample Selection}
These approaches aim to identify correctly labeled examples from a noisy training dataset. Based on the dynamics of DNNs learning with label noise, small-loss trick which regards small-loss examples as correctly labeled ones is widely used. After selecting correct labels, some approaches~\cite{coteaching, jocor} directly remove wrong labeled examples and train DNNs on the remained data, while others~\cite{dividemix, scanmix} only discard wrong labels but preserve the corresponding instances, then they leverage semi-supervised learning to train DNNs. To reduce the accumulated error caused by incorrect selection, this type of approaches usually maintains multiple DNNs or refines the selected set iteratively, which increases the memory burden and computational cost remarkably.

\subsection{Regularization}
Many regularization methods introduce regularizers into loss functions. To name a few, the work of~\cite{peer} randomly pair an instance and another label to construct a peer sample, then uses a peer regularizer to punish DNNs from overly agreeing with the peer sample. Generalized Jenson-Shannon Divergence (GJS)~\cite{js} introduces a consistency regularizer forcing DNNs to make consistent predictions given two augmented versions of a single input. Early learning regularization (ELR)~\cite{elr} encourages the predictions of DNNs to agree with the exponential moving average of the past predictions. However, it is important to note that these regularization methods come with additional memory burden or computational costs, unlike robust loss functions. There are also other types of regularizations such as contrastive learning~\cite{sscl}, model pruning~\cite{prune}, over-parameterization~\cite{sop}, and so on.

\subsection{Robust Loss Function}
While the commonly used CE easily overfits label noise due to its strong fitting ability, although MAE is theoretically noise-tolerant~\cite{mae}, it suffers from underfitting due to its poor fitting ability. Subsequently, a large amount of work improves generalization by reconciling fitting ability with robustness in different ways. While GCE~\cite{gce} is an interpolation between CE and MAE, symmetric cross entropy (SCE)~\cite{sce} equals a convex combination of CE and MAE, and active passive loss~\cite{apl} just replaces CE in SCE with normalized CE. Taylor cross entropy~\cite{tce} realizes an interpolation between CE and MAE through Taylor Series, while the work of~\cite{js} scales the Jensen-Shannon Divergence to construct the interpolation. Robust loss functions cause no changes to the training process, they require no extra information such as the noise rate or a clean validation set, and incur no additional memory burden or computational cost.While the commonly used CE easily overfits label noise due to its strong fitting ability, although MAE is theoretically noise-tolerant~\cite{mae}, it suffers from underfitting due to its poor fitting ability. Subsequently, a large amount of work improves generalization by reconciling fitting ability with robustness in different ways. While GCE~\cite{gce} is an interpolation between CE and MAE, symmetric cross entropy (SCE)~\cite{sce} equals a convex combination of CE and MAE, and active passive loss~\cite{apl} just replaces CE in SCE with normalized CE. Taylor cross entropy~\cite{tce} realizes an interpolation between CE and MAE through Taylor Series, while the work of~\cite{js} scales the Jensen-Shannon Divergence to construct the interpolation. Robust loss functions cause no changes to the training process, they require no extra information such as the noise rate or a clean validation set, and incur no additional memory burden or computational cost.

\section{Method} \label{sec_method}

\subsection{Preliminary}
\noindent \textbf{Risk minimization} \quad We consider a typical $k$-class classification problem where $k \geq 2$. Denote the feature space by $\mathcal{X} \subset \mathbb{R}^d$ and the class space by $\mathcal{Y} = [k] = \{1,...k\}$. The classifier $\arg \max_{i} f_i(\cdot)$ is a function that maps feature space to class space, where $f: \mathcal{X} \rightarrow \mathcal{C}$, $\mathcal{C} \subseteq [0, 1]^k$, $\forall \mathbf{c} \in \mathcal{C}$, $\mathbf{1}^T\mathbf{c}=1$. In this paper, we consider the common case where $f$ is a DNN with a softmax output layer. For brevity, we call $f$ as the classifier in the following. Without label noise, the training data $\{\mathbf{x}_i, y_i\}_{i=1}^N$ is drawn i.i.d. from distribution $p(\mathbf{x}, y)$ over $\mathcal{X} \times \mathcal{Y}$. Given a loss function $L: \mathcal{C} \times \mathcal{Y} \rightarrow \mathbb{R}_+$ and a classifier $f$, the $L$ risk of $f$ is defined as
\begin{equation}
	R_L(f)  = \mathbb{E}_{p(\mathbf{x},y)} [L(f(\mathbf{x}), y)]  = \mathbb{E}_{p(\mathbf{x})} [\sum_y L(f(\mathbf{x}), y) p(y|\mathbf{x})], \label{eq_rm}
\end{equation}
where $\mathbb{E}$ represents expectation. In the following, we denote $\sum_y L(f(\mathbf{x}), y) p(y|\mathbf{x})$ by $R_L(f(\mathbf{x}))$. Under the risk minimization framework, our objective is to learn $f^*_L = \arg \min_{f} R_L(f)$. In the presence of label noise, we can only access the noisy training data $\{\mathbf{x}_i, \tilde{y}_i\}_{i=1}^N$ drawn i.i.d. from distribution $\tilde{p}(\mathbf{x}, \tilde{y})$. In this case, the $L$ risk of $f$ is defined as
\begin{equation}
	\tilde{R}_L(f) = \mathbb{E}_{\tilde{p}(\mathbf{x},\tilde{y})} [L(f(\mathbf{x}), \tilde{y})] = \mathbb{E}_{p(\mathbf{x})} [\sum_{\tilde{y}} L(f(\mathbf{x}), \tilde{y}) \tilde{p}(\tilde{y}|\mathbf{x})]. \label{eq_nrm}
\end{equation}
Similarly, we denote $\sum_{\tilde{y}} L(f(\mathbf{x}), \tilde{y}) \tilde{p}(\tilde{y}|\mathbf{x})$ by $\tilde{R}_L(f(\mathbf{x}))$ and $\arg \min_{f} \tilde{R}_L(f)$ by $\tilde{f}^*_L$.

\noindent \textbf{Label noise model} \quad The most generic label noise is termed instance-dependent noise, where the noise depends on both features and labels, so there is $\tilde{p}(\mathbf{x},\tilde{y}) = \sum_y \tilde{p}(\tilde{y}|\mathbf{x},y) p(\mathbf{x},y)$. By contrast, asymmetric noise assumes that the noise depends only on labels, i.e. $\tilde{p}(\tilde{y}|\mathbf{x},y) = \tilde{p}(\tilde{y}|y)$. In addition, the most ideal label noise is called symmetric noise, where each true label is flipped into other labels with equal probability, i.e. $\tilde{p}(\tilde{y}|\mathbf{x},y)$ is a constant.

\subsection{Gradient Analysis} \label{sec:gradient}

The widely-used loss functions CE and MAE are defined as
\begin{gather}
	L_{\text{CE}} (f(\mathbf{x}), y) = -\log f_y(\mathbf{x}), \\
	L_{\text{MAE}} (f(\mathbf{x}), y) = 1 - f_y(\mathbf{x}).
\end{gather}
Although the recently proposed robust loss functions vary from each other in formulation, many of them can be seen as interpolations between CE and MAE. For instance, GCE~\cite{gce}, TCE~\cite{tce} and JS~\cite{js} are respectively defined as
\begin{gather}
	L_{\text{GCE}} (f(\mathbf{x}), y) = \frac{1 - f_y^q(\mathbf{x})}{q}, \\
	L_{\text{TCE}}(f(\mathbf{x}),y) = \sum_{i=1}^{t} \frac{(1 - f_y(\mathbf{x}))^i}{i}, \\
	L_{\text{JS}}(f(\mathbf{x}),y) = \frac{\pi_1 D_{\text{KL}} (\mathbf{e}^{(y)} || \mathbf{m}) + (1-\pi_1) D_{\text{KL}} (f(\mathbf{x}) || \mathbf{m})}{-(1-\pi_1) \log(1-\pi_1)},
\end{gather}
where $\mathbf{e}^{(y)}$ denotes a one-hot vector with $\mathbf{e}^{(t)}_j=1$ iff $j=t$, $\mathbf{m} = \pi_1 \mathbf{e}^{(y)} + (1-\pi_1) f(\mathbf{x})$ and $D_{KL}(\cdot || \cdot)$ denotes KL divergence. GCE equals CE when $q \rightarrow 0$ and equals MAE when $q=1$, TCE equals CE when $q \rightarrow +\infty$ and equals MAE when $t=1$, JS equals CE when $\pi_1 \rightarrow 0$ and equals MAE when $\pi_1 \rightarrow 1$. 	

To intuitively show the behavior of different loss functions on different examples, we perform gradient analyses. First, the gradients of loss functions mentioned above can be calculated as
\begin{gather}
	\frac{\partial L_{\text{CE}} (f(\mathbf{x}), y)}{\partial \boldsymbol{\theta}}  = -\frac{1}{f_y(\mathbf{x})} \nabla_{\boldsymbol{\theta}}f_y(\textbf{x}), \\
	\frac{\partial L_{\text{MAE}} (f(\mathbf{x}), y)}{\partial \boldsymbol{\theta}} = -\nabla_{\boldsymbol{\theta}}f_y(\textbf{x}), \\
	\frac{\partial L_{\text{GCE}} (f(\mathbf{x}), y)}{\partial \boldsymbol{\theta}} = -f_y^{q-1}(\mathbf{x}) \nabla_{\boldsymbol{\theta}}f_y(\textbf{x}), \\
	\frac{\partial L_{\text{TCE}} (f(\mathbf{x}), y)}{\partial \boldsymbol{\theta}} = - \sum_{i=1}^{t} (1 - f_y(\mathbf{x}))^{i-1} \nabla_{\boldsymbol{\theta}}f_y(\textbf{x}), \\
	\frac{\partial L_{\text{JS}} (f(\mathbf{x}), y)}{\partial \boldsymbol{\theta}} = - \frac{(1-\pi_1) \log (1 + \frac{\pi_1}{(1 - \pi_1) f_y(\mathbf{x})})}{(1-\pi_1) \log(1-\pi_1)} \nabla_{\boldsymbol{\theta}}f_y(\textbf{x}), \label{gradient_js}
\end{gather}
where $\boldsymbol{\theta}$ is the set of parameters of $f$. A small $f_y(\textbf{x})$ means that the prediction on $\textbf{x}$ is not consistent with the given labels, i.e., $(\textbf{x}, y)$ is a hard example. We plot the gradient of loss w.r.t the $y$-th posterior probability (i.e. the coefficient of $\nabla_{\boldsymbol{\theta}}f_y(\textbf{x})$ on the right side of equations above) in Figure~\ref{fig_gradient}. CE puts heavy emphasis on hard examples, resulting in significant overfitting in the presence of label noise since these hard examples may be those with incorrect labels. In contrast, MAE gives equal weight to all examples, effectively avoiding overfitting but leading to severe underfitting. Robust loss functions strike a balance between CE and MAE by giving less emphasis to hard examples than CE does and more attention to them than MAE does. By properly tuning the hyperparameters, robust loss functions can achieve improved generalization.

\begin{figure}[t]
	\centering
	\includegraphics[width= \textwidth]{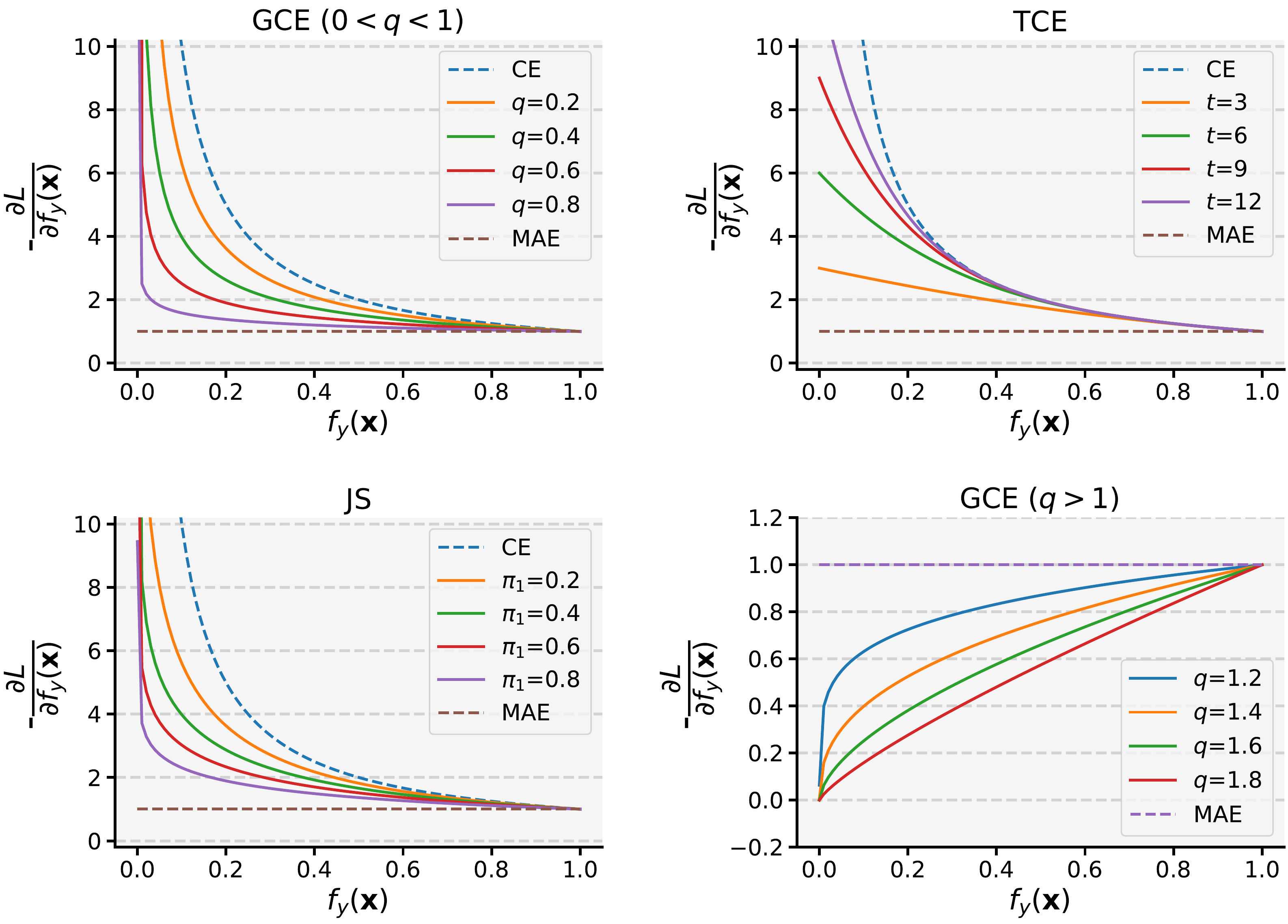}
	\caption{Gradient of loss w.r.t the $y$-th posterior probability.}
	\label{fig_gradient}
\end{figure}

\subsection{Dynamics-Aware Loss} \label{sec:dal}
Although robust loss functions outperform both CE and MAE substantially, there exists a clear discrepancy between these static losses and the dynamics of DNNs learning with label noise. Specifically, the work of~\cite{dynamic} empirically proved that DNNs first memorize the correct labels and then memorize the wrong labels, so static trade-off between fitting ability and robustness results in inferior performance. Also, static losses are sensitive to hyperparameters, with slight changes potentially leading to serious underfitting or overfitting.

According to the dynamics of DNNs learning with label noise, we think it is reasonable to prioritize fitting ability at the early stage of the training process when DNNs learn beneficial patterns, and then increase the weight of robustness gradually because DNNs overfit label noise thereafter. Taking GCE as an example, we gradually increase $q$ which controls the trade-off between fitting ability and robustness rather than fix it during the whole training process. For simplicity, we let $q$ increase linearly from $q_s$ to $q_e$. We call GCE with a dynamic $q$ dynamic GCE (DGCE) in the following. 

Based on the gradient analyses in Figure~\ref{fig_gradient}, smaller $q$ provides stronger fitting ability. Consequently, we set $q_s$ to a small value for a rapid performance increase at the early stage. Moreover, we think the range of $q$ should not be limited within (0,1) as the vanilla GCE does. As shown in righ-bottom subfigure of Figure~\ref{fig_gradient}, with $q>1$ GCE puts more emphasis on easy examples. In fact, GCE with $q>1$ plays a similar role to some reweighting methods~\cite{discussion2}. The main difference is that the latter explicitly assign more weights to correctly labeled examples that are typically identified by meta DNNs, which are usually trained on a clean dataset without label noise, while the former implicitly puts more emphasis on examples whose model predictions are more consistent with the provided labels. According to the widely-used small-loss trick, these examples are more likely to be correctly labeled. Therefore, DGCE with $q_e>1$ can further reduce the negative impact of label noise and guarantee a steady performance growth at the later stage. Unless otherwise specified, we always set $q_e=1.5$ and only tune $q_s$ in the following. In this way, DGCE is relatively insensitive to $q_s$ because $q_e>1$ relives DNNs from overfitting. 

On the one hand, since some other robust loss functions such as TCE and JS can also be regarded as interpolations between CE and MAE, the dynamic rule can also provide them with performance boosts. On the other hand, only GCE can pay more attention to easy examples than hard ones by extending the range of its hyper-parameter, so DGCE outperforms other dynamic losses. We present theoretical analyses in Section~\ref{sec_theory} to verify the merit of DGCE and defer the empirical evidence to Section~\ref{sec:integration}.

Although DGCE overcomes drawbacks of static loss functions, we observe that it may suffers from underfitting at the later stage of the training process. In Eq.~(\ref{eq_nrm}), $\tilde{R}_L(f)$ is in the form of $\mathbb{E}_{p(\mathbf{x})} [P+Q]$. We define $P=\tilde{p}(\tilde{y}^*|\mathbf{x}) L(f(\mathbf{x}), \tilde{y}^*)$ where $\tilde{y}^* = \arg \max_{\tilde{y}} \tilde{p}(\tilde{y}|\mathbf{x})$, $Q=\sum_{\tilde{y} \neq \tilde{y}^*}\tilde{p}(\tilde{y}|\mathbf{x}) L(f(\mathbf{x}), \tilde{y})$. The Q term may decrease even when P is fixed, which implies that $f_{\tilde{y}^*}(\mathbf{x})$ may be fixed when the loss decreases, especially if the noise rate is relatively high, i.e., $\tilde{p}(\tilde{y}^*|\mathbf{x})$ is relatively small. To solve this problem, we introduce a bootstrapping term
\begin{equation}
	L_{\text{BS}}(f(\mathbf{x})) = - \log \max_y f_y(\mathbf{x})
\end{equation}
when $q>1$. Intuitively, since $f$ has been trained on $\tilde{p}(\mathbf{x}, \tilde{y})$ with DGCE with $q<1$ during the early stage, at this point $f_{\tilde{y}^*}(\mathbf{x}) > f_{\tilde{y}}(\mathbf{x}), \forall \tilde{y} \neq \tilde{y}^*$ holds with a high probability, so $L_{\text{BS}}$ encourages an explicit increase of $f_{\tilde{y}^*}(\mathbf{x})$. To make the training process stable, we gradually increase the weight of the bootstrapping term $\lambda$ from 0 to $\lambda_e$. Unless otherwise specified, we always set $\lambda_e=1.0$ in the following.

\begin{algorithm}[t]
	\caption{Dynamics-Aware Loss}
	\label{al_dal}
	\KwIn{Data $\{(\mathbf{x}_i, \tilde{y}_i)\}_{i=1}^n$, $T$, $q_s$, $q_e$ (1.5 by default), $\lambda_e$ (1.0 by default)}
	\KwOut{Classifier $f$}
	$t_0 = \frac{1 - q_s}{q_e - q_s} T$ \\
	\For{$t=1:T$}
	{
		$q(t) = q_s + (q_e - q_s)\frac{t}{T}$ \\
		\eIf{$t<t_0$ or $t_0 > T$}{$\lambda(t)=0$}{$\lambda(t)=\lambda_e\frac{t-t_0}{T-t_0}$}
		Minimize $L(f)=\frac{1}{n} \sum_i (\frac{1-f_{\tilde{y}}^{q(t)}(\mathbf{x}_i)}{q(t)} + \lambda(t) \frac{- \log \max_y f_y(\mathbf{x}_i)}{q(t)\log k}) $
	}
\end{algorithm}	

Dynamics-aware loss (DAL) is the combination of DGCE and $L_{\text{BS}}$, we display the pseudocode of DAL in Algorithm~\ref{al_dal}. As the range of DGCE is $[0, \frac{1}{q(t)}]$ and that of $L_{\text{BS}}$ is $[0, \log k]$ (since $\max_y f_y(\mathbf{x}) > \frac{1}{k}$ always holds for a $k$-class classification problem), we divide $L_{\text{BS}}$ by $q(t)\log k$ such that its range is the same as that of DGCE. Obviously, our proposed DAL requires no extra information and incurs no additional memory burden or computational cost.

\section{Theoretical Analyses} \label{sec_theory}
In this section, we provide some theoretical guarantees on the performance of DAL. Since $f$ is a DNN, according to the universal approximation theorem~\cite{universal}, we suppose that $f^*_L$ minimizes $R_L(f(\mathbf{x}))$ for each $\mathbf{x}$.

\begin{theorem} \label{theorem:1}
	With $0<q<1, f^*_{\mathrm{GCE}}(\mathbf{x}) = (\frac{p(1|\mathbf{x})^{\frac{1}{1-q}}}{\sum_i p(i|\mathbf{x})^{\frac{1}{1-q}}},..., \frac{p(k|\mathbf{x})^{\frac{1}{1-q}}}{\sum_i p(i|\mathbf{x})^{\frac{1}{1-q}}})$.
\end{theorem}

\begin{proof}
	According to Eq.~(\ref{eq_rm}), the optimization problem can be formulated as
	\begin{gather}
		\text{minimize} \quad \sum_y p(y|\mathbf{x}) \frac{1 - f_y^q(\mathbf{x})}{q}, \\
		\text{s.t.} \quad \sum_y f_y(\mathbf{x}) - 1 = 0.
	\end{gather}
	The Lagrangian function can be formulated as
	\begin{equation}
		L(f(\mathbf{x}), \lambda) = \sum_y p(y|\mathbf{x}) \frac{1 - f_y^q(\mathbf{x})}{q} + \lambda (\sum_y f_y(\mathbf{x}) - 1).
	\end{equation}
	Let $\frac{\partial L(f(\mathbf{x}), \lambda)}{\partial f_y(\mathbf{x})} = 0$, we can derive
	\begin{equation}
		\frac{f_y^{1-q} (\mathbf{x})}{p(y|\mathbf{x})} = \frac{1}{\lambda}.
	\end{equation}
	Combining the constrain that $\sum_y f_y(\mathbf{x}) = 1$, we can conclude the proof.
\end{proof}

\begin{theorem} \label{theorem:2}
	$\forall q>1, f^*_{\mathrm{GCE}}(\mathbf{x}) = \mathbf{e}^{(y^*)}$, where $y^* = \arg \max_y p(y|\mathbf{x})$, $\mathbf{e}^{(y)}$ denotes a one-hot vector with $\mathbf{e}^{(t)}_j=1$ iff $j=t$.
\end{theorem}

\begin{proof}
	With $q>1$, we have
	\begin{equation}
		\sum_y p(y|\mathbf{x}) \frac{1 - f_y^q(\mathbf{x})}{q} \geq \frac{1 - \sum_y p(y|\mathbf{x}) f_y(\mathbf{x})}{q} \geq \frac{1 - p(y^* | \mathbf{x})}{q} ,
	\end{equation}
	with equality holds iff $f(\mathbf{x}) = \mathbf{e}^{(y^*)}$. This completes the proof.
\end{proof}

\begin{corollary} \label{corollary:1}
	$\forall q>1, \lambda >0, f^*_{\mathrm{GCE+BS}}(\mathbf{x}) = \mathbf{e}^{(y^*)}$.
\end{corollary}

\begin{proof}
	Based on Theorem 2, $f(\mathbf{x}) = \mathbf{e}^{(y^*)}$ minimizes $R_{\text{GCE}}(f(\mathbf{x}))$. Besides, any one-hot vector minimizes $L_{\text{BS}}(f(\mathbf{x}))$. Therefore, the above corollary holds.
\end{proof}

\begin{remark}
	Theorem~\ref{theorem:1} indicates that the prediction of $f^*_{\text{GCE}}$ with $q<1$ depends on $q$ while Corollary~\ref{corollary:1} indicates that the prediction of $f^*_{\text{GCE+BS}}$ with any $q>1, \lambda>0$ is always a one-hot vector which is independent of both $q$ and $\lambda$. Therefore, under risk minimization framework, minimizing DAL with $q_e=1.5,\lambda_e=1.0$ derives the same classifier as minimizing GCE+BS with any $q>1, \lambda>0$ finally because at the later stage of the training process, DAL with $q_e=1.5,\lambda_e=1.0$ always has $q>1$ and $\lambda>0$. That is, the theoretical results for GCE+BS with $q>1, \lambda>0$ can directly apply to DAL with $q_e=1.5,\lambda_e=1.0$.
\end{remark}

\begin{definition} \label{definition:1}
	For a multi-class classification classifier, we define 0-1 loss as
	\begin{equation}
		L_{0\raisebox{0mm}{-}1} (f(\mathbf{x}), y) = \mathds{1}  [\arg\max_{y'} f_{y'}(\mathbf{x}) \neq y],
	\end{equation}
	where $\mathds{1} [\cdot]$ is the indicator function. The classifier minimizing 0-1 risk is called Bayes optimal classifier and the corresponding risk is called Bayes error rate.
\end{definition}

\begin{theorem} \label{theorem:3}
	$\forall q>1, \lambda>0,$ given a classifier $f$, we have
	\begin{equation}
		R_{0\raisebox{0mm}{-}1}(f) - R_{0\raisebox{0mm}{-}1}(f^*_{0\raisebox{0mm}{-}1}) \leq 1 - \mathbb{E}_{p(\mathbf{x})} \mathds{1} [(y^* = \tilde{y}^*) \land (f_{\tilde{y}^*}(\mathbf{x}) > \tilde{f}^*_{\mathrm{GCE+BS}\tilde{y}^*}(\mathbf{x}) - \frac{1}{2})], \label{eq_main}
	\end{equation}
	where $y^* = \arg \max_y p(y|\mathbf{x})$ and $\tilde{y}^* = \arg \max_{\tilde{y}} \tilde{p} (\tilde{y}|\mathbf{x})$.
\end{theorem}

\begin{proof}
	According to Corollary~\ref{corollary:1}, $\tilde{f}^*_{\text{GCE+BS}\tilde{y}^*}(\mathbf{x}) = 1$, so if $f_{\tilde{y}^*}(\mathbf{x}) > \tilde{f}^*_{\text{GCE+BS}\tilde{y}^*}(\mathbf{x}) - \frac{1}{2} = \frac{1}{2}$, we have $\arg \max_{\tilde{y}} f_{\tilde{y}}(\mathbf{x}) = \tilde{y}^*$. Furthermore, if $y^* = \tilde{y}^*$, we have $\arg \max_{\tilde{y}} f_{\tilde{y}}(\mathbf{x}) = y^*=\arg \max_{y} f^*_{0\raisebox{0mm}{-}1y}(\mathbf{x})$. Therefore, we have $R_{0\raisebox{0mm}{-}1}(f(\mathbf{x})) = R_{0\raisebox{0mm}{-}1}(f^*_{0\raisebox{0mm}{-}1}(\mathbf{x}))$ if $y^* = \tilde{y}^*$ and $f_{\tilde{y}^*}(\mathbf{x}) > \tilde{f}^*_{\text{GCE+BS}\tilde{y}^*}(\mathbf{x}) - \frac{1}{2}$, which concludes the proof.
\end{proof}

\begin{remark} \label{remark:2}
	In practice, it is almost impossible to derive $\tilde{f}^*_{\mathrm{GCE+BS}}$ from finite training data. Theorem~\ref{theorem:3} gives an upper bound of the difference between the error rate caused by a classifier $f$ (rather than $\tilde{f}^*_{\mathrm{GCE+BS}}$) and Bayes error rate. Different from previous theoretical results~\cite{mae,gce,tce}, Theorem~\ref{theorem:3} makes no assumption about the label noise model, for instance, the label noise is symmetric or asymmetric. The upper bound is determined by both the data distribution (whether $y^* = \tilde{y}^*$ holds) and the classifier (whether $f_{\tilde{y}^*}(\mathbf{x}) > \tilde{f}^*_{\text{GCE+BS}\tilde{y}^*}(\mathbf{x}) - \frac{1}{2}$ holds, which is a sufficient condition for $\arg \max_{\tilde{y}} f_{\tilde{y}}(\mathbf{x}) = \tilde{y}^*$ because $\tilde{f}^*_{\text{GCE+BS}\tilde{y}^*}(\mathbf{x})=1$). If we replacing $\tilde{f}^*_{\mathrm{GCE+BS}}$ with $q>1,\lambda>0$ with $\tilde{f}^*_{\mathrm{GCE}}$ with $0<q<1$ in Theorem~\ref{theorem:3}, to make it hold, the gap between $f_{\tilde{y}^*}(\mathbf{x})$ and $\tilde{f}^*_{\text{GCE}\tilde{y}^*}(\mathbf{x})$ should be smaller than that between $f_{\tilde{y}^*}(\mathbf{x})$ and $\tilde{f}^*_{\text{GCE+BS}\tilde{y}^*}(\mathbf{x})$ (i.e. $\frac{1}{2}$). Because if $\tilde{p}(\tilde{y}^*|\mathbf{x}) < 1$, we have $\tilde{f}^*_{\text{GCE}\tilde{y}^*}(\mathbf{x}) < 1$ according to Theorem~\ref{theorem:1}. Furthermore, higher noise rate may lead to smaller $\tilde{p}(\tilde{y}^*|\mathbf{x})$, resulting in smaller $\tilde{f}^*_{\text{GCE}\tilde{y}^*}(\mathbf{x})$. Therefore, we cannot conclude that $\arg \max_{\tilde{y}} f_{\tilde{y}}(\mathbf{x}) = \tilde{y}^*$ with $f_{\tilde{y}^*}(\mathbf{x}) > \tilde{f}^*_{\text{GCE}\tilde{y}^*}(\mathbf{x}) - \frac{1}{2}$. Consequently, the error rate of the classifier trained with GCE+BS with $q>1$ and $\lambda > 0$ tends to be lower than that trained with GCE with $0<q<1$.
\end{remark}

\section{Experiement} \label{sec_experiment}

In this section, we first introduce the experimental setup (Section~\ref{sec:setup}). Then we perform a comprehensive comparison of DAL with many baselines on various benchmarks (Section~\ref{sec:comparison}) and sensitivity analyses for DAL (Section~\ref{sec:sensitivity}). After that, a case study shows why DAL can outperform other loss functions (Section~\ref{sec:case}). Finally, we empirically verify that DAL is complementary to other robust algorithms for learning with label noise (Section~\ref{sec:integration}) and it helps to improve backdoor robustness (Section~\ref{sec:backdoor}).

\subsection{Setup} \label{sec:setup}

\textbf{Datasets} \quad We perform experiments on CIFAR-10 and CIFAR-100 with synthetic label noise and two real-world noisy datasets Animal-10N~\cite{animal10n} and Webvision~\cite{webvision}. For CIFAR, we use three types of synthetic label noise, including
\begin{itemize}
	\item symmetric noise. The labels are resampled from a uniform distribution over all labels with probability $\eta$, where $\eta$ is the noise rate.
	
	\item asymmetric noise. For CIFAR-10, the labels are changed as follows with probability $\eta$: truck $\rightarrow$ automobile, bird $\rightarrow$ airplane, cat $\leftrightarrow$ dog and deer $\rightarrow$ horse; for CIFAR-100, the labels are cycled to the next sub-class of the same super-class, e.g. the labels of super-class ``vehicle1" are modified as follows: bicycle $\rightarrow$ bus $\rightarrow$ motorcycle $\rightarrow$ pickup truck $\rightarrow$ train $\rightarrow$ bicycle.
	
	\item instance-dependent noise. For both CIFAR-10 and CIFAR-100, the generation follows Algorithm 2 in~\cite{instance}.\footnote{We also provide performance on other types of instance-dependent noise in the Appendix.}
	
\end{itemize}
For Webvision, we follow the ``Mini" setting in~\cite{apl} which takes only the first 50 classes of the Google resized images as the training dataset. Then we evaluate the classification performance on the same 50 classes of Webvision validation set.

\begin{table*}[t]
	\caption{Optimal hyper-parameters of robust loss functions under our setup.}
	\label{tab_hyperparam}
	\setlength\tabcolsep{6pt}
	\centering
	\renewcommand{\arraystretch}{0.9}
	\scalebox{0.8}{
		\begin{tabular}{lccccc}
			\toprule
			Method & Hyper-parameter & CIFAR-10 & CIFAR-100 & Animal-10N & WebVision \\
			\midrule
			GCE~\cite{gce} & ($q$) & (0.9) & (0.7) & - & - \\
			SCE~\cite{sce} & ($\alpha$, $\beta$) & (0.1, 10.0) & (6.0, 1.0) & (6.0, 1.0) & (10.0, 1.0) \\
			NLNL~\cite{nlnl} & ($N$) & (1) & (110) & - & - \\
			BTL~\cite{btl} & ($t_1$, $t_2$) & (0.7, 3.0) & (0.7, 1.5) & - & - \\
			NCE+RCE~\cite{apl} & ($\alpha$, $\beta$) & (1.0, 0.1) & (10.0, 0.1) & (10.0, 0.1) & (50.0, 0.1) \\
			TCE~\cite{tce} & ($t$) & (3) & (18) & - & - \\
			NCE+AGCE~\cite{agce} & ($\alpha$, $\beta$, $a$, $q$) & (1.0, 0.4, 6, 1.5) & (10.0, 0.1, 3, 3) & (0.0, 1.0, 1e-5, 0.7) & (0.0, 1.0, 1e-5, 0.5)\\
			CE+SR~\cite{sr} & ($\tau$, $\lambda_0$, $r$, $p$, $\rho$) & (0.5, 1.5, 1, 0.1, 1.02) & (0.5, 8.0, 1, 0.01, 1.02) & (0.5, 8.0, 1, 0.01, 1.02) & (0.5, 2.0, 1, 0.01, 1.02) \\
			JS~\cite{js} & ($\pi$) & (0.9) & (0.5) & (0.5) & (0.1) \\
			Poly-1 & ($\epsilon$) & (10) & (2) & - & - \\
			\midrule
			DAL & ($q_s$, $q_e$, $\lambda_e$) & (0.8, 1.5, 1.0) & (0.6, 1.5, 1.0) & (0.6, 1.5, 1.0) & (0.4, 1.5, 1.0) \\
			\bottomrule
		\end{tabular}
	}
\end{table*}	

\textbf{Baselines} \quad We first compare DAL with several robust loss functions. These baselines include Generalized Cross Entropy (GCE)~\cite{gce}, Negative Learning for Noisy Labels (NLNL)~\cite{nlnl}, Symmetric Cross Entropy (SCE)~\cite{sce}, Bi-Tempered Logistic Loss (BTL)~\cite{btl}, Normalized Cross Entropy with Reverse Cross Entropy (NCE+RCE)~\cite{apl}, Taylor Cross Entropy (TCE)~\cite{tce}, Normalized Cross Entropy with Asymmetric Generalized Cross Entropy (NCE+AGCE)~\cite{agce}, Cross Entropy with Sparse Regularization (CE+SR)~\cite{sr}, Jensen-Shannon Divergence Loss (JS)~\cite{js} and Poly-1~\cite{poly}. Similar to DAL, they all require no additional information and incur no extra memory burden or computational costs.

In addition, for those more complicated approaches which cannot be compared with DAL directly, we verify that DAL can provide them with performance boosts, we integrate DAL with the following methods: early learning regularization (ELR)~\cite{elr}, generalized Jensen-Shannon Divergence (GJS)~\cite{js}, Co-teaching~\cite{coteaching}, JoCoR~\cite{jocor} and DivideMix~\cite{dividemix}.

\textbf{Experimental details} \quad When comparing DAL with robust loss functions, we use a single shared learning setup for different methods for fair comparison. For CIFAR and Animal-10N, we train a ResNet18~\cite{resnet} (the first 7×7 Conv of stride 2 is replaced with 3×3 Conv of stride 1, and the first max pooling layer is also removed for CIFAR.)\footnote{We also provide performance on other DNNs in the Appendix.} using SGD for 150 epochs with momentum 0.9, weight decay $10^{-4}$, batch size 128, initial learning rate 0.01, and cosine learning rate annealing. We also apply typical data augmentations including random crop and horizontal flip. For Webvision, we train a ResNet50~\cite{resnet} using SGD for 250 epochs with nesterov momentum 0.9, weight decay $3 \times 10^{-5}$, batch size 512, and initial learning rate 0.4. The learning rate is multiplied by 0.97 after each epoch. Typical data augmentations including random crop, color jittering, and horizontal flip are applied to both Webvision. When integrating DAL with other approaches, for each approach we use the training setup reported in official codes except that we set the number of epochs to 200 for DivideMix and GJS for convenience. Besides, we use a single set of hyper-parameters $(\pi_1, \pi_2, \pi_3)=(0.3, 0.35, 0.35)$ for GJS rather than tune hyper-parameters elaborately for different label noise. We provide our source codes in https://github.com/XiuchuanLi/DAL.

\subsection{Comparison with Robust Losses} \label{sec:comparison}

\begin{table}[t] 
	\caption{Test accuracies (\%) on CIFAR-10 with label noise. The results (mean±std) are reported over 3 random run. The best results are boldfaced and underlined, the second best results are boldfaced, while the third best results are underlined.}
	\vskip 0.05in
	\label{tab_cifar10}
	\setlength\tabcolsep{7pt}
	\centering
	\renewcommand{\arraystretch}{1.0}
	\scalebox{0.9}{
		\begin{tabular}{lcccccccc}
			\toprule
			\multirow{2}{*}{Method} & \multicolumn{4}{c}{Symmetric} & \multicolumn{2}{c}{Asymmetric} & \multicolumn{2}{c}{Instance-dependent} \\
			\cmidrule(lr){2-5} \cmidrule(lr){6-7} \cmidrule(lr){8-9}
			& 20\% & 40\% & 60\% & 80\% & 20\% & 40\% & 20\% & 40\% \\
			\midrule
			CE & 83.30\scriptsize±0.19 & 67.85\scriptsize±0.53 & 47.79\scriptsize±0.42 & 25.80\scriptsize±0.20 & 86.00\scriptsize±0.15 & 74.98\scriptsize±0.09 & 80.86\scriptsize±0.11 & 61.40\scriptsize±0.28 \\
			GCE~\cite{gce} & 90.68\scriptsize±0.08 & 87.33\scriptsize±0.15 & 81.29\scriptsize±0.28 & 61.93\scriptsize±0.24 & 88.96\scriptsize±0.15 & 58.79\scriptsize±0.14 & \underline{89.40\scriptsize±0.17} & \underline{81.21\scriptsize±0.13} \\
			SCE~\cite{sce} & 89.16\scriptsize±0.32 & 85.41\scriptsize±0.22 & 77.81\scriptsize±0.53 & 51.46\scriptsize±1.92 & 87.71\scriptsize±0.41 & \textbf{77.40\scriptsize±0.13} & 88.11\scriptsize±0.17 & 77.80\scriptsize±0.74 \\
			NLNL~\cite{nlnl} & 82.18\scriptsize±0.44 & 73.25\scriptsize±0.23 & 60.79\scriptsize±0.96 & 43.46\scriptsize±0.16 & 85.23\scriptsize±0.22 & \underline{\textbf{81.91\scriptsize±0.22}} & 81.77\scriptsize±0.42 & 69.27\scriptsize±0.30 \\
			BTL~\cite{btl} & 89.93\scriptsize±0.30 & 78.22\scriptsize±0.24 & 58.00\scriptsize±0.20 & 29.01\scriptsize±0.50 & 86.51\scriptsize±0.14 & 74.58\scriptsize±0.43 & 86.58\scriptsize±0.41 & 64.72\scriptsize±0.74 \\
			NCE+RCE~\cite{apl} & 90.70\scriptsize±0.03 & 87.02\scriptsize±0.11 & 81.16\scriptsize±0.16 & 64.62\scriptsize±0.18 & 88.51\scriptsize±0.13 & \underline{76.42\scriptsize±0.23} & 88.92\scriptsize±0.13 & 77.80\scriptsize±0.50 \\
			TCE~\cite{tce} & 90.50\scriptsize±0.11 & 86.30\scriptsize±0.24 & 77.42\scriptsize±0.13 & 46.91\scriptsize±0.08 & 87.87\scriptsize±0.21 & 58.07\scriptsize±0.28 & 88.88\scriptsize±0.18 & 73.15\scriptsize±0.26 \\
			NCE+AGCE~\cite{agce} & 90.60\scriptsize±0.13 & 87.28\scriptsize±0.32 & 81.34\scriptsize±0.04 & 66.50\scriptsize±0.73 & \underline{88.99\scriptsize±0.08} & 76.65\scriptsize±0.45 & 89.10\scriptsize±0.13 & 78.77\scriptsize±0.18 \\
			CE+SR~\cite{sr} & \underline{\textbf{91.39\scriptsize±0.14}} & \underline{87.82\scriptsize±0.30} & \underline{82.71\scriptsize±0.34} & \textbf{69.70\scriptsize±0.29} & 88.26\scriptsize±0.24 & 74.95\scriptsize±0.10 & 88.28\scriptsize±0.49 & 68.86\scriptsize±0.63 \\
			JS~\cite{js} & 90.62\scriptsize±0.06 & 86.28\scriptsize±0.28 & 77.04\scriptsize±0.54 & 43.04\scriptsize±0.34 & 88.30\scriptsize±0.16 & 68.73\scriptsize±0.50 & 88.27\scriptsize±0.56 & 73.92\scriptsize±0.33 \\
			Poly-1~\cite{poly} & 84.89\scriptsize±0.08 & 68.25\scriptsize±0.67 & 47.87\scriptsize±0.56 & 25.14\scriptsize±0.32 & 86.19\scriptsize±0.20 & 75.52\scriptsize±0.42 & 81.37\scriptsize±0.23 & 61.50\scriptsize±0.60 \\
			\hline
			DGCE & \underline{90.92\scriptsize±0.13} & \textbf{87.88\scriptsize±0.15} & \textbf{82.97\scriptsize±0.16} & \underline{66.18\scriptsize±0.26} & \textbf{89.66\scriptsize±0.23} & 72.16\scriptsize±0.11 & \textbf{89.78\scriptsize±0.10} & \textbf{82.99\scriptsize±0.33} \\
			DAL & \textbf{90.95\scriptsize±0.09} & \underline{\textbf{88.73\scriptsize±0.10}} & \underline{\textbf{84.28\scriptsize±0.16}} & \underline{\textbf{72.51\scriptsize±0.20}} & \underline{\textbf{90.08\scriptsize±0.10}} & 73.34\scriptsize±0.18 & \underline{\textbf{90.33\scriptsize±0.08}} & \underline{\textbf{85.04\scriptsize±0.08}} \\
			\bottomrule
		\end{tabular}
	}
\end{table}

\begin{table}[t] 
	\caption{Test accuracies (\%) on CIFAR-100 with label noise. The results (mean±std) are reported over 3 random run. The best results are boldfaced and underlined, the second best results are boldfaced, while the third best results are underlined.}
	\vskip 0.05in
	\label{tab_cifar100}
	\setlength\tabcolsep{7pt}
	\centering
	\renewcommand{\arraystretch}{1.0}
	\scalebox{0.9}{
		\begin{tabular}{lcccccccc}
			\toprule
			\multirow{2}{*}{Method} & \multicolumn{4}{c}{Symmetric} & \multicolumn{2}{c}{Asymmetric} & \multicolumn{2}{c}{Instance-dependent} \\
			\cmidrule(lr){2-5} \cmidrule(lr){6-7} \cmidrule(lr){8-9}
			& 20\% & 40\% & 60\% & 80\% & 20\% & 40\% & 20\% & 40\% \\
			\midrule
			CE & 60.41\scriptsize±0.19 & 43.78\scriptsize±0.44 & 25.03\scriptsize±0.13 & 8.45\scriptsize±0.32 & 60.89\scriptsize±0.25 & 43.74\scriptsize±0.51 & 60.73\scriptsize±0.36 & 45.31\scriptsize±0.36 \\
			GCE~\cite{gce} & 68.37\scriptsize±0.27 & 61.94\scriptsize±0.44 & 49.91\scriptsize±0.25 & 22.22\scriptsize±0.51 & 62.21\scriptsize±0.32 & 41.19\scriptsize±0.89 & \underline{67.22\scriptsize±0.28} & 54.37\scriptsize±0.29 \\
			SCE~\cite{sce} & 58.66\scriptsize±0.55 & 42.98\scriptsize±0.26 & 24.86\scriptsize±0.32 & 8.56\scriptsize±0.16 & 59.70\scriptsize±0.32 & 43.06\scriptsize±0.16 & 59.36\scriptsize±0.27 & 44.19\scriptsize±0.16 \\
			NLNL~\cite{nlnl} & 57.89\scriptsize±0.66 & 44.05\scriptsize±0.47 & 26.14\scriptsize±0.22 & 11.80\scriptsize±0.33 & 57.33\scriptsize±0.20 & 38.82\scriptsize±0.14 & 57.23\scriptsize±0.55 & 42.96\scriptsize±0.46 \\
			BTL~\cite{btl} & 61.83\scriptsize±0.13 & 47.54\scriptsize±0.16 & 30.47\scriptsize±0.54 & 13.73\scriptsize±0.38 & 58.70\scriptsize±0.14 & 42.91\scriptsize±0.32 & 59.31\scriptsize±0.40 & 44.18\scriptsize±0.44 \\
			NCE+RCE~\cite{apl} & 67.47\scriptsize±0.14 & 61.27\scriptsize±0.16 & \underline{51.02\scriptsize±0.12} & \underline{25.60\scriptsize±0.28} & 63.17\scriptsize±0.26 & 43.47\scriptsize±0.43 & 65.83\scriptsize±0.30 & 54.21\scriptsize±0.48 \\
			TCE~\cite{tce} & 63.97\scriptsize±0.65 & 57.40\scriptsize±0.63 & 41.46\scriptsize±0.67 & 15.12\scriptsize±0.27 & 54.97\scriptsize±0.49 & 39.73\scriptsize±0.19 & 59.42\scriptsize±0.40 & 36.48\scriptsize±1.38 \\
			NCE+AGCE~\cite{agce} & 67.06\scriptsize±0.22 & 60.93\scriptsize±0.17 & 49.09\scriptsize±0.35 & 20.10\scriptsize±0.61 & \underline{64.87\scriptsize±0.29} & \underline{\textbf{46.87\scriptsize±0.28}} & 66.38\scriptsize±0.46 & \underline{56.35\scriptsize±0.10} \\
			CE+SR~\cite{sr} & \underline{68.84\scriptsize±0.23} & \underline{62.03\scriptsize±0.39} & 50.28\scriptsize±0.25 & 10.75\scriptsize±0.20 & 59.16\scriptsize±0.62 & 41.80\scriptsize±0.23 & 63.19\scriptsize±0.06 & 47.45\scriptsize±0.51 \\
			JS~\cite{js} & 67.58\scriptsize±0.52 & 61.01\scriptsize±0.31 & 47.95\scriptsize±0.38 & 20.03\scriptsize±0.27 & 59.67\scriptsize±0.89 & 41.23\scriptsize±0.50 & 65.31\scriptsize±0.41 & 49.12\scriptsize±1.11 \\
			Poly-1~\cite{poly} & 60.13\scriptsize±0.16 & 44.20\scriptsize±0.65 & 25.84\scriptsize±0.15 & 8.44\scriptsize±0.27 & 60.81\scriptsize±0.22 & 43.63\scriptsize±0.60 & 60.76\scriptsize±0.45 & 45.65\scriptsize±0.03 \\
			\hline
			DGCE  & \textbf{68.89\scriptsize±0.26} & \textbf{63.71\scriptsize±0.19} & \textbf{53.97\scriptsize±0.44} & \textbf{30.90\scriptsize±0.44} & \textbf{65.80\scriptsize±0.15} & \underline{44.87\scriptsize±0.14} & \textbf{68.01\scriptsize±0.13} & \textbf{58.92\scriptsize±0.37} \\
			DAL & \underline{\textbf{69.00\scriptsize±0.09}} & \underline{\textbf{64.80\scriptsize±0.17}} & \underline{\textbf{56.72\scriptsize±0.27}} & \underline{\textbf{34.33\scriptsize±0.61}} & \underline{\textbf{66.60\scriptsize±0.04}} & \textbf{45.36\scriptsize±0.42} & \underline{\textbf{68.27\scriptsize±0.20}} & \underline{\textbf{60.38\scriptsize±0.54}} \\
			\bottomrule
		\end{tabular}
	}
\end{table}

\begin{table}[t] \small
	\caption{Test accuracies (\%) on real-word noisy datasets Animal-10N and WebVision. The best results are boldfaced and underlined, the second best results are boldfaced, while the third best results are underlined.}
	\vskip 0.05in
	\centering
	\renewcommand{\arraystretch}{0.9}
	\setlength\tabcolsep{9pt}
	\label{tab_real}
	\scalebox{0.9}{
		\begin{tabular}{ccccccccc}
			\toprule
			Dataset & CE & SCE~\cite{sce} & NCE+RCE~\cite{apl} & NCE+AGCE~\cite{agce} & CE+SR~\cite{sr} & JS~\cite{js} & DGCE & DAL \\
			\midrule
			Animal-10N & 80.60 & 81.62 & 80.92 & 81.40 & \underline{81.82} & 81.14 & \textbf{82.62} & \underline{\textbf{82.66}} \\
			WebVision & 64.04 & 65.60 & 64.64 & \underline{68.88} & 68.72 & 64.96 & \textbf{69.92} & \underline{\textbf{70.04}} \\
			\bottomrule
		\end{tabular}
	}
\end{table}   

We randomly select 10\% examples from the noisy training set as the validation set for hyper-parameter tuning, then use the best hyper-parameters to train DNNs on the full training set. The optimal hyper-parameters of robust loss functions in our setup are summarized in Table~\ref{tab_hyperparam}. The experimental results on CIFAR-10 and CIFAR-100 are respectively shown in Table~\ref{tab_cifar10} and Table~\ref{tab_cifar100}. On CIFAR-10, DAL achieves the best performance in 6 settings and DGCE achieves the second best performance in 5 settings. On CIFAR-100, DAL reaches the highest accuracy in 7 settings and DGCE also reaches the second highest accuracy in 7 setting. Overall, the performance gaps between DAL and DGCE and that between DAL and any other baseline are larger with higher noise rate. The former is because underfitting is more serious with higher noise rate as stated in Section~\ref{sec:dal}, in which case BS provides substantial performance boosts. The latter is consistent with Remark~\ref{remark:2} which states that the superiority of DAL over GCE is more remarkable with higher noise rate.

The comparison on real-world noisy datasets is shown in Table~\ref{tab_real}, where both DAL and DGCE also overtake baselines. In both datasets, DAL is only slightly better than DGCE and outperforms other baselines by a relatively small margin. This is because the noise rate of Animal-10N is only about 8\%~\cite{animal10n} and that of Webvision is only about 20\%~\cite{webvision}.

\subsection{Sensitivity Analyses} \label{sec:sensitivity}
We perform sensitivity analyses for DAL and report the results in Figure~\ref{fig_sensitivity}, we also provide the results of GCE for comparison. Obviously, DAL is far more insensitive to $q_s$. For example, on CIFAR-10 with 40\% instance-dependent noise, Varying $q_s$ of DAL in [0.75, 0.85] only leads to at most about 1\% performance fluctuation while varying $q$ of GCE in [0.85, 0.95] results in 8\% performance gap. Besides, we also notice that DAL with inferior $q_s$ always outperforms GCE with optimal $q$. 

As stated above, for convenience we always set $q_e$ to 1.5 and $\lambda_e$ to 1.0 and only tune $q_s$. In fact, setting $q_e$ or $\lambda_e$ to other value leads to similar performance. Fixing $q_s$=0.6, we show the performance of DAL with different $q_e$ or $\lambda_e$ in Figure~\ref{fig_sensitivity}. Obviously, increasing or decreasing $q_e$ or $\lambda_e$ substantially (by 0.3 or 0.5) derives only less than 1\% performance fluctuation.

\begin{figure}[!htbp]
	\centering
	\includegraphics[width= \textwidth]{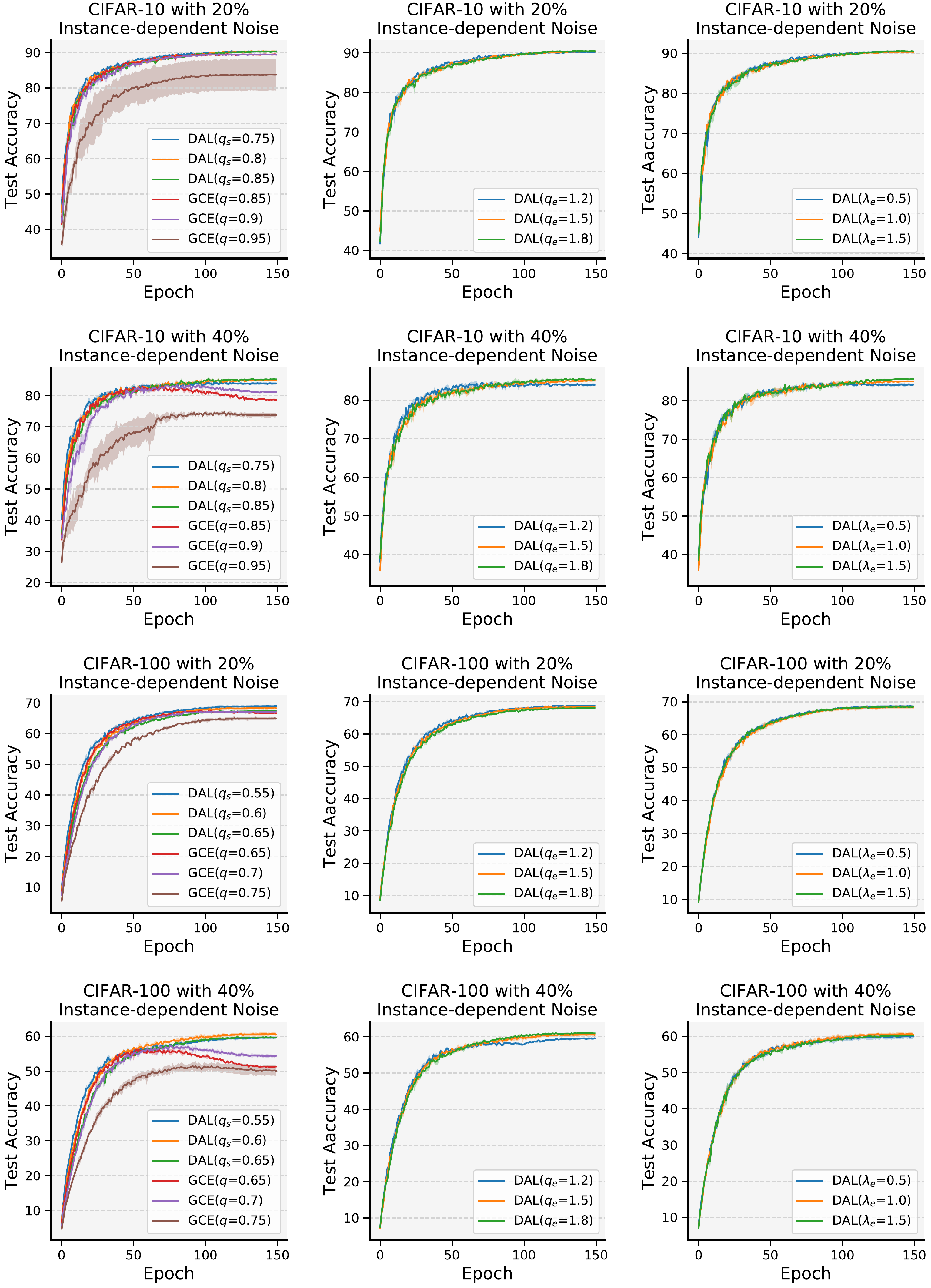}
	\caption{Sensittivity analysis with instance-dependent noise.}
	\label{fig_sensitivity}
\end{figure}

\begin{figure}[!htbp]
	\centering
	\includegraphics[width= \textwidth]{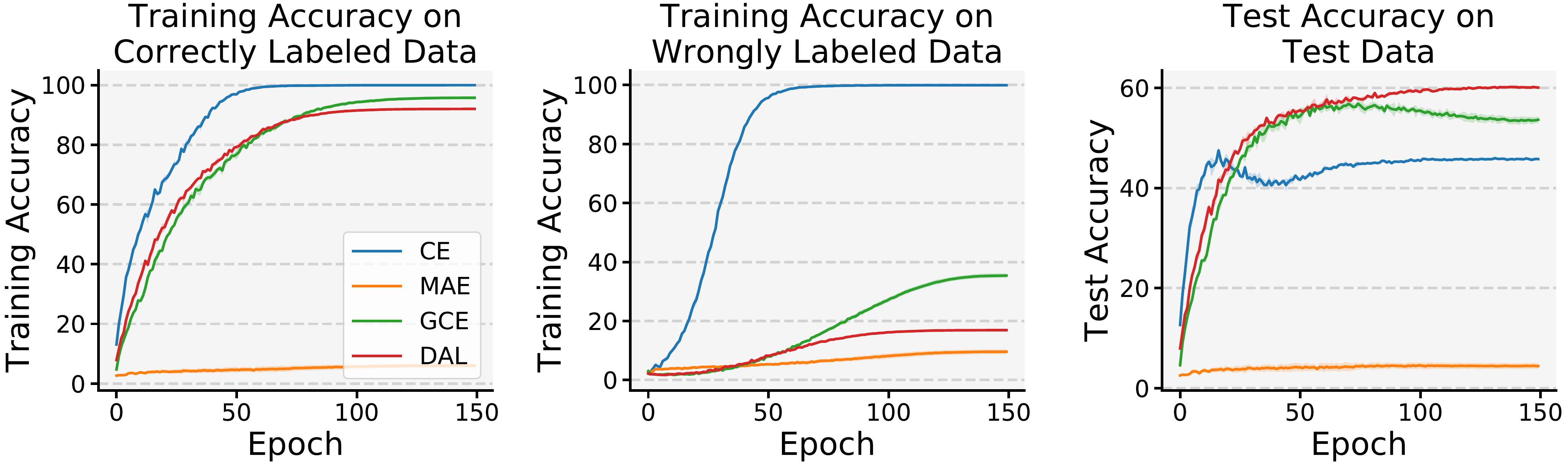}
	\caption{Training process on CIFAR-100 with 40\% instance-dependent noise.}
	\label{fig_process}
\end{figure}

\subsection{Case study} \label{sec:case}
We delve deep into the training process with different losses. In Figure~\ref{fig_process}, besides the test accuracy, we also display the training accuracy on both correctly labeled examples and wrongly labeled examples during the training process. The latter is calculated based on the provided labels rather than the ground truth. The model trained with CE can achieve 100\% training accuracy on both correctly and wrongly labeled data, indicating that it overfits label noise. In contrast, the training accuracy of DNNs trained with MAE is rather low ($<$10\%) on both correctly and wrongly labeled data, which means that it suffers from serious underfitting, leading to a quite poor generalization. By the way, we also observe that the training accuracy on correctly labeled data increases fast at the outset while that on wrongly labeled data rises slowly initially and grows quickly after several epochs, which is consistent with the dynamics of DNNs learning with label noise.

On the one hand, The model trained with GCE can achieve high training accuracy (compared to CE) on correctly labeled data finally while the figure on wrongly labeled data is relatively low (compared to MAE), so it outperforms DNNs trained with both CE and MAE on test data. On the other hand, its training accuracy on wrongly labeled data still rises remarkably at the late stage, resulting in a decrease of generalization.

As for DNNs trained with DAL, since DAL provides DNNs with strong fitting ability at the early stage when DNNs tend to learn beneficial patterns, both the training accuracy on correctly labeled data and test accuracy rise rapidly. Afterwards, as DAL gradually improves robustness, it prevents DNNs from overfitting label noise at the later stage, so the training accuracy on wrongly labeled data remains at a low level and the test accuracy grows steadily throughout the whole training process.

\begin{table}[t] 
	\caption{Performance before and after intergrating DAL on CIFAR-100.}
	\vskip 0.05in
	\label{tab_integrate}
	\setlength\tabcolsep{6.0pt}
	\centering
	\renewcommand{\arraystretch}{0.85}
	\scalebox{0.86}{
		\begin{tabular}{lcccccccc}
			\toprule
			\multirow{2}{*}{Method} & \multicolumn{4}{c}{Symmetric} & \multicolumn{2}{c}{Asymmetric} & \multicolumn{2}{c}{Instance} \\
			\cmidrule(lr){2-5} \cmidrule(lr){6-7} \cmidrule(lr){8-9}
			& 20\% & 40\% & 60\% & 80\% & 20\% & 40\% & 20\% & 40\% \\
			\midrule
			TCE~\cite{tce} & 63.97 & 57.40 & 41.46 & 15.12 & 54.97 & 39.73 & 59.42 & 36.48 \\
			+DAL & 64.41\scriptsize(\textcolor{red}{+0.44}) & 60.38\scriptsize(\textcolor{red}{+2.98}) & 47.31\scriptsize(\textcolor{red}{+5.85}) & 20.05\scriptsize(\textcolor{red}{+4.93}) & 56.00\scriptsize(\textcolor{red}{+1.03}) & 40.96\scriptsize(\textcolor{red}{+1.23}) & 61.55\scriptsize(\textcolor{red}{+2.13}) & 41.68\scriptsize(\textcolor{red}{+5.20}) \\
			\midrule
			JS~\cite{js} & 67.58 & 61.01 & 47.95 & 20.03 & 59.67 & 41.23 & 65.31 & 49.12 \\
			+DAL & 68.28\scriptsize(\textcolor{red}{+0.70}) & 62.82\scriptsize(\textcolor{red}{+1.71}) & 51.08\scriptsize(\textcolor{red}{+3.13}) & 24.03\scriptsize(\textcolor{red}{+4.00}) & 61.49\scriptsize(\textcolor{red}{+1.82}) & 41.64\scriptsize(\textcolor{red}{+0.41}) & 66.70\scriptsize(\textcolor{red}{+1.39}) & 51.70\scriptsize(\textcolor{red}{+2.58}) \\
			\midrule
			ELR~\cite{elr} & 72.64 & 69.53 & 63.33 & 28.55 & 74.01 & 68.12 & 73.62 & 71.86 \\
			+DAL & 73.64\scriptsize(\textcolor{red}{+1.00}) & 70.94(\scriptsize\textcolor{red}{+1.41}) & 65.91(\scriptsize\textcolor{red}{+2.58}) & 34.16(\scriptsize\textcolor{red}{+5.61}) & 73.83(\scriptsize\textcolor{olive}{-0.18}) & 67.22(\scriptsize\textcolor{olive}{-0.90}) & 73.95(\scriptsize\textcolor{red}{+0.33}) & 71.56(\scriptsize\textcolor{olive}{-0.30}) \\
			\midrule
			GJS~\cite{js} & 75.91 & 72.40 & 62.83 & 33.16 & 74.01 & 56.33 & 75.53 & 68.00 \\
			+DAL & 76.03\scriptsize(\textcolor{red}{+0.12}) & 72.88\scriptsize(\textcolor{red}{+0.44}) & 64.70\scriptsize(\textcolor{red}{+1.87}) & 38.98\scriptsize(\textcolor{red}{+5.82}) & 75.42\scriptsize(\textcolor{red}{+1.41}) & 61.73\scriptsize(\textcolor{red}{+5.40}) & 76.17\scriptsize(\textcolor{red}{+0.64}) & 71.47\scriptsize(\textcolor{red}{+3.47}) \\
			\midrule
			Co-teaching~\cite{coteaching} & 66.07 & 59.83 & 48.47 & 20.21 & 64.01 & 46.67 & 65.24 & 53.11 \\
			+DAL & 67.34\scriptsize(\textcolor{red}{+1.27}) & 60.15\scriptsize(\textcolor{red}{+0.32}) & 51.04\scriptsize(\textcolor{red}{+2.57}) & 26.16\scriptsize(\textcolor{red}{+5.95}) & 63.81\scriptsize(\textcolor{olive}{-0.20}) & 47.36\scriptsize(\textcolor{red}{+0.69}) & 64.79\scriptsize(\textcolor{olive}{-0.45}) & 54.73\scriptsize(\textcolor{red}{+1.62}) \\
			\midrule
			JoCoR~\cite{jocor} & 63.29 & 59.50 & 51.60 & 27.94 & 56.91 & 40.60 & 61.25 & 51.58 \\
			+DAL & 64.70\scriptsize(\textcolor{red}{+1.41}) & 60.36\scriptsize(\textcolor{red}{+0.86}) & 53.07\scriptsize(\textcolor{red}{+1.47}) & 31.86\scriptsize(\textcolor{red}{+3.92}) & 59.12\scriptsize(\textcolor{red}{+2.21}) & 41.90\scriptsize(\textcolor{red}{+1.30}) & 62.05\scriptsize(\textcolor{red}{+0.80}) & 51.81\scriptsize(\textcolor{red}{+0.23}) \\
			\midrule
			DivideMix~\cite{dividemix} & 76.23 & 74.11 & 69.23 & 52.38 & 75.18 & 52.86 & 76.10 & 68.51 \\
			+DAL & 77.07\scriptsize(\textcolor{red}{+0.84}) & 74.68\scriptsize(\textcolor{red}{+0.57}) & 70.38\scriptsize(\textcolor{red}{+1.15}) & 54.91\scriptsize(\textcolor{red}{+2.53}) & 75.42\scriptsize(\textcolor{red}{+0.24}) & 54.60\scriptsize(\textcolor{red}{+1.74}) & 75.97\scriptsize(\textcolor{olive}{-0.18}) & 68.93\scriptsize(\textcolor{red}{+0.42}) \\
			
			\bottomrule
		\end{tabular}
	}
\end{table}

\subsection{Integrating DAL with Other Approaches} \label{sec:integration}
In this section, we integrate our proposed DAL with other robust algorithms for learning with label noise. As we stated in Section~\ref{sec:gradient}, besides GCE, TCE and JS can also be seen as interpolations between CE and MAE. TCE is equivalent to CE when $t \rightarrow +\infty$ and MAE when $t=1$, JS equals CE when $\pi_1 \rightarrow 0$ and MAE when $\pi_1 \rightarrow 1$. For TCE we let $t$ decrease from 20 to 1, for JS we let $\pi_1$ increase from 0 to 1. On the one hand, the dynamic interpolation provides both TCE and JS with remarkable performance gains. On the other hand, unlike GCE which can put more emphasis on easy examples than hard ones by extending the range of its hyper-parameter, JS and TCE always pay more attention to hard examples, so the improved TCE and JS are both inferior to DGCE or DAL.

Besides robust losses incurring no extra memory burden or computational cost, we also integrate our proposed DAL with other more sophisticated algorithms. ELR~\cite{elr} is a combination of CE and early learning regularizer, the latter must maintain a label distribution for each sample, which increases memory burden substantially when the class space is large. GJS~\cite{js} is a combination of JS and consistency regularizer, the latter requires strong data augmentation and doubles the back-and-forth computational cost. We replace CE in ELR and JS in GJS with DAL. Co-teaching~\cite{coteaching} and JoCoR~\cite{jocor} are both sample selection methods. They maintain two DNNs with the same architecture simultaneously who select potentially correct samples for each other. We replace CE in both methods with DAL.  DivideMix~\cite{dividemix} is a quite sophisticated method which employs semi-supervised learning. It has a warmup process at the outset, we replace CE with DAL during this period, then use examples whose predictions agree with their labels to train DNNs with CE for another several epochs. We set $(q_s, q_e)$ to (0.0, 1.0) for all approaches. Since DAL is combined with more powerful techniques such as sample selection and label correction, the labels gradually become far more accurate than the provided ones, in this case $q_e > 1$ degenerates performance due to underfitting. It is clear from Table~\ref{tab_integrate} that DAL can provide different methods with performance boosts consistently even if we do not tune $(q_s,q_e)$ elaborately. Especially with higher noise rate, the performance boosts are more remarkable. For instance, under 80\% symmetric noise, DAL improves Co-teaching by about 6\% and even provides the sophisticated DivideMix with about 3\% performance gain without requiring any extra information or resources.

\subsection{Robustness against Backdoor Attacks} \label{sec:backdoor}

\begin{figure}[t]
	\centering
	\includegraphics[width= 0.8 \textwidth]{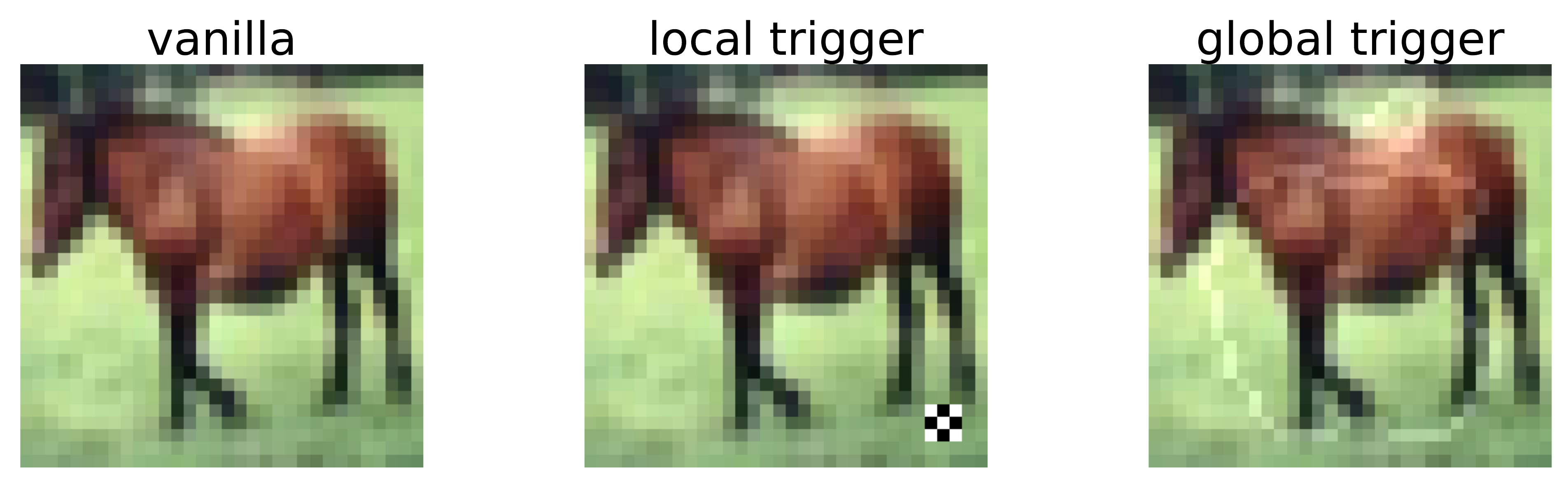}
	\caption{Illustration of backdoor triggers.}
	\label{fig_trigger}
\end{figure}

\begin{table}[t]
	\caption{Backdoor Attack on DNNs trained with different loss functions. The results (mean±std) are reported over 3 random run.}
	\vskip 0.05in
	\label{tab_backdoor}
	\setlength\tabcolsep{4.0pt}
	\centering
	\renewcommand{\arraystretch}{1}
	\scalebox{0.9}{
		\begin{tabular}{llcccc}
			\toprule
			\multirow{2}{*}{Dataset} & \multirow{2}{*}{Method} & \multicolumn{2}{c}{Local Trigger} & \multicolumn{2}{c}{Global Trigger} \\
			\cmidrule(lr){3-4} \cmidrule(lr){5-6}
			& & Test Accuracy & Success Rate & Test Accuracy & Success Rate \\
			\midrule
			\multirow{3}{*}{CIFAR-10} & CE & 93.39\scriptsize±0.13 & 99.91\scriptsize±0.03 & 93.46\scriptsize±0.10 & 68.16\scriptsize±2.63 \\
			& GCE & 92.41\scriptsize±0.13 & 0.68\scriptsize±0.71 & 92.37\scriptsize±0.07 & 38.55\scriptsize±2.35 \\
			& DAL & 92.42\scriptsize±0.22 & 0.23\scriptsize±0.05 & 92.33\scriptsize±0.16 & 29.50\scriptsize±1.84 \\
			\midrule
			\multirow{3}{*}{CIFAR-100} & CE & 73.95\scriptsize±0.16 & 96.07\scriptsize±0.81 & 73.99\scriptsize±0.07 & 49.36\scriptsize±2.09 \\
			& GCE & 71.32\scriptsize±0.20 & 98.84\scriptsize±0.73 & 71.81\scriptsize±0.17 & 29.45\scriptsize±4.35 \\
			& DAL & 72.06\scriptsize±0.26 & 0.40\scriptsize±0.35 & 72.18\scriptsize±0.10 & 16.95\scriptsize±7.75 \\
			\bottomrule
		\end{tabular}
	}
\end{table}
Besides label noise, we observe that DAL also helps to improve robustness against backdoor attacks. Typical backdoor attacks~\cite{backdoor1} inject triggers into a small part of training examples and change their labels into a specific target class, such that the target model performs well on benign samples whereas consistently classifies any input containing the backdoor trigger into the target class. Following~\cite{trigger}, we use two types of backdoor triggers as shown in Figure~\ref{fig_trigger}, which are added into 100 randomly selected training examples whose labels are converted into the target class. We first feed clean test samples into the target model to calculate the test accuracy, then remove test samples which are classified into the target class and add triggers into all remained test images to compute the backdoor success rate.

As shown in Table~\ref{tab_backdoor}, test accuracies of DNNs trained with different losses are comparable to each other. However, DAL exhibits stronger robustness against backdoor attacks than both CE and GCE. On CIFAR-100, local triggers yield almost 100\% success rate on DNNs trained with CE and GCE while 0\% success rate on that trained with DAL. Based on the dynamics of DNNs, they firstly learn patterns shared by most training examples and eventually memorize the correlation between backdoor triggers and target class because the poisoned examples only account for a small proportion (0.2\%). Since DAL reduces fitting ability gradually, it also helps to improve backdoor robustness.

\section{Conclusion} \label{sec_conclusion}
In this paper, we propose Dynamics-Aware Loss (DAL) to address the discrepancy between the static robust loss functions and the dynamics of DNNs learning with label noise. At the early stage, since DNNs tend to learn beneficial patterns, DAL prioritizes fitting ability to achieve high test accuracy quickly. Subsequently, as DNNs overfit label noise gradually, DAL improves the weight of robustness thereafter. Moreover, DAL puts more emphasis on easy examples than hard ones and introduces a bootstrapping term at the later stage to further combat label noise. We provide detailed theoretical analyses and extensive experimental results to demonstrate the superiority of DAL over existing robust loss functions. To the best of our knowledge, DAL is the first robust loss function that takes DNNs' dynamics into account, which both has theoretical guarantee and achieves remarkable performance. It serves as a simple and strong baseline for learning with label noise.

The main limitation of DAL is that it is not flexible enough. Specifically, DAL
lets $q$ increase from $q_s$ to $q_e$ which are both
manually assigned rather than adjusts $q$ adaptively according to
DNNs' learning status during the training process. The challenge lies in the
lack of a good quantitative indicator reflecting DNNs' learning status. The
local intrinsic dimensionality has been used in~\cite{dynamic2}, but it can
only characterize DNNs' behavior on a specific training example and estimating
the local intrinsic dimensionality for each training example incurs great
computational cost. Besides designing more flexible robust losses, this work
could also be extended in another direction. In this paper we only empirically
verify that DAL helps to improve backdoor robustness, but learning with label
noise also shares some similarities with other types of weakly supervised
learning. For example, partial-label learning~\cite{partial} assumes that each
example is assigned with a set of candidate labels, of which only one is the
ground truth while others are wrong labels. Therefore, the dynamics of DNNs
learning with label noise may also inspire more advanced algorithms handling
other weakly supervised problems.

\bibliography{mybibfile}
\bibliographystyle{unsrt}

\end{document}